\documentclass[runningheads]{llncs}

\usepackage{eccv}

\usepackage{eccvabbrv}

\usepackage{graphicx}
\usepackage{booktabs}

\usepackage[accsupp]{axessibility}  

\usepackage{hyperref}
\usepackage{amsmath}
\usepackage{amssymb}
\usepackage{placeins}
\usepackage{enumitem}
\usepackage{orcidlink}
\newcommand{\Stwo}{\mathbb{S}^2}
\newcommand{\SOthree}{\mathrm{SO}(3)}

\newcommand{\R}{\mathbb{R}}
\newcommand{\pp}{\,pp}
\newcommand{\CVF}{\rho}  

\begin{document}

\title{Coordinate Singularities Break Conformal Coverage for Gaze and Head Pose} \titlerunning{Coordinate Singularities Break Conformal Coverage} 

\author{ 
Mohammadreza Jamalifard\orcidlink{0009-0008-4969-1398} 
\and Yaxiong Lei\orcidlink{0000-0002-0697-7942}$^*$ \and Parastoo Azizinezhad\orcidlink{0009-0008-2537-5059} \and Javier Andreu-Perez\orcidlink{0000-0002-7421-4808}$^*{\dagger}$ } \authorrunning{M. Jamalifard et al.} \institute{Smart Health Technologies Group, Centre for Computational Intelligence, School of Computer Science and Electronic Engineering, University of Essex, Colchester, UK\\ \email{\{m.jamalifard,yaxiong.lei,p.azizinezhad,j.andreu-perez\}@essex.ac.uk}\\[2pt] $^*$Equal contribution.\quad $^\dagger$Corresponding author. 
}


\maketitle

\begin{abstract}
Conformal prediction provides distribution-free reliability guarantees for vision systems, but these guarantees depend on how prediction errors are measured in the output space.
Many vision tasks produce outputs on curved spaces (\eg gaze directions on the sphere or 3D head rotations), yet intermediate prediction heads, residuals, uncertainty estimates, or conformal scores are often defined in flat coordinate charts such as yaw--pitch or Euler angles.
We show that this scoring choice introduces systematic geometric distortion near coordinate singularities (large pitch angles on the sphere and poses approaching gimbal lock in 3D rotations). Across four datasets (ETH-XGaze, Gaze360, BIWI, AFLW2000-3D), slice-conditional coverage at a nominal 90\% target drops by 30--50 percentage points in these regions, falling to 38.9\% on ETH-XGaze and 42.0\% on Gaze360 at gaze pitch above $70^\circ$, and to 57.5\% on BIWI and 55.2\% on AFLW2000-3D at head pose pitch above $60^\circ$ near gimbal lock, despite marginal coverage remaining near 90\%.
We prove that this is structural. Scalar thresholding changes the size of chart-coordinate prediction sets but leaves their distorted axis ratios unchanged.
To diagnose this hidden failure mode, we show that a simple geometric quantity, the Riemannian volume density, strongly correlates with where coverage collapse occurs.
Finally, we show that coordinate-free geodesic scoring removes this distortion. 
It requires no retraining and adds negligible computational cost.
\end{abstract}

\section{Introduction}
\label{sec:intro}

Consider a driver monitoring system using conformal prediction to guarantee 90\% coverage of a driver's gaze direction. Marginal conformal evaluation appears nominal: overall coverage is 90.8\%. Yet at extreme pitch angles associated with distraction and drowsiness~\cite{nikan2022appearance}, coverage drops to roughly 39\%, so fewer than half of the prediction sets contain the true gaze direction. Such viewing angles are not rare: combined eye and cervical rotation regularly exceeds $70^\circ$~\cite{lee2019differences,feipel1999normal} during mirror checks, instrument-panel glances, and drowsiness-related head drops, consistent with natural driving variation~\cite{fridman2016owl}. The failure persists across architectures, training procedures, and calibration set sizes.
Its source is more fundamental: the coordinate system used for the conformal nonconformity score.

Gaze directions lie on the sphere $\Stwo$~\cite{zhang2020eth}, and
head poses are rotations in $\SOthree$~\cite{hartley2003multiple,ma2012invitation}.
Vision pipelines often represent intermediate predictions, labels, or downstream calibration residuals using yaw--pitch or Euler angles, even when final benchmark accuracy is reported with angular error.
Our analysis concerns the nonconformity score used inside conformal calibration. If this score is computed from chart-coordinate residuals, the resulting prediction sets inherit the chart geometry. These charts distort distances: they compress regions near coordinate singularities (the poles of $\Stwo$ or gimbal lock in $\SOthree$) and stretch others~\cite{huynh2009metrics,zhou2019continuity}.
A fixed threshold corresponds to a large region near the equator
but only a narrow sliver near the pole
(Fig.~\ref{fig:gaze_sphere}).
Because marginal coverage is enforced by construction,
coverage is redistributed across the manifold.
Well-conditioned regions overcover while singular regions
undercover~\cite{barber2021limits,vovk2005algorithmic}.
Global evaluation metrics therefore do not reveal this failure.

A natural response is to apply adaptive conformal
methods~\cite{lei2018distribution,romano2019conformalized},
which rescale thresholds based on local difficulty.
However, these methods address heteroscedastic model error,
not geometric distortion.
The mismatch is one of \emph{shape}, not \emph{scale}.
Near the pole, yaw--pitch residuals form an ellipse whose
eccentricity is fixed by the coordinate geometry.
Scalar rescaling can change the size of this ellipse but not
its shape.
On ETH-XGaze~\cite{zhang2020eth}, normalised yaw--pitch scoring
improves near-pole coverage only from 39\% to 50\%.
Switching to geodesic scoring with the same normalisation
raises coverage to 89\%. These observations suggest that conformal prediction on manifold-valued outputs requires respecting the intrinsic geometry of the output space.

\paragraph{Contributions.}
We formalise this geometric failure mode and show how to
eliminate it.

\begin{figure*}[t!]
    \centering
    \includegraphics[width=0.93\textwidth]{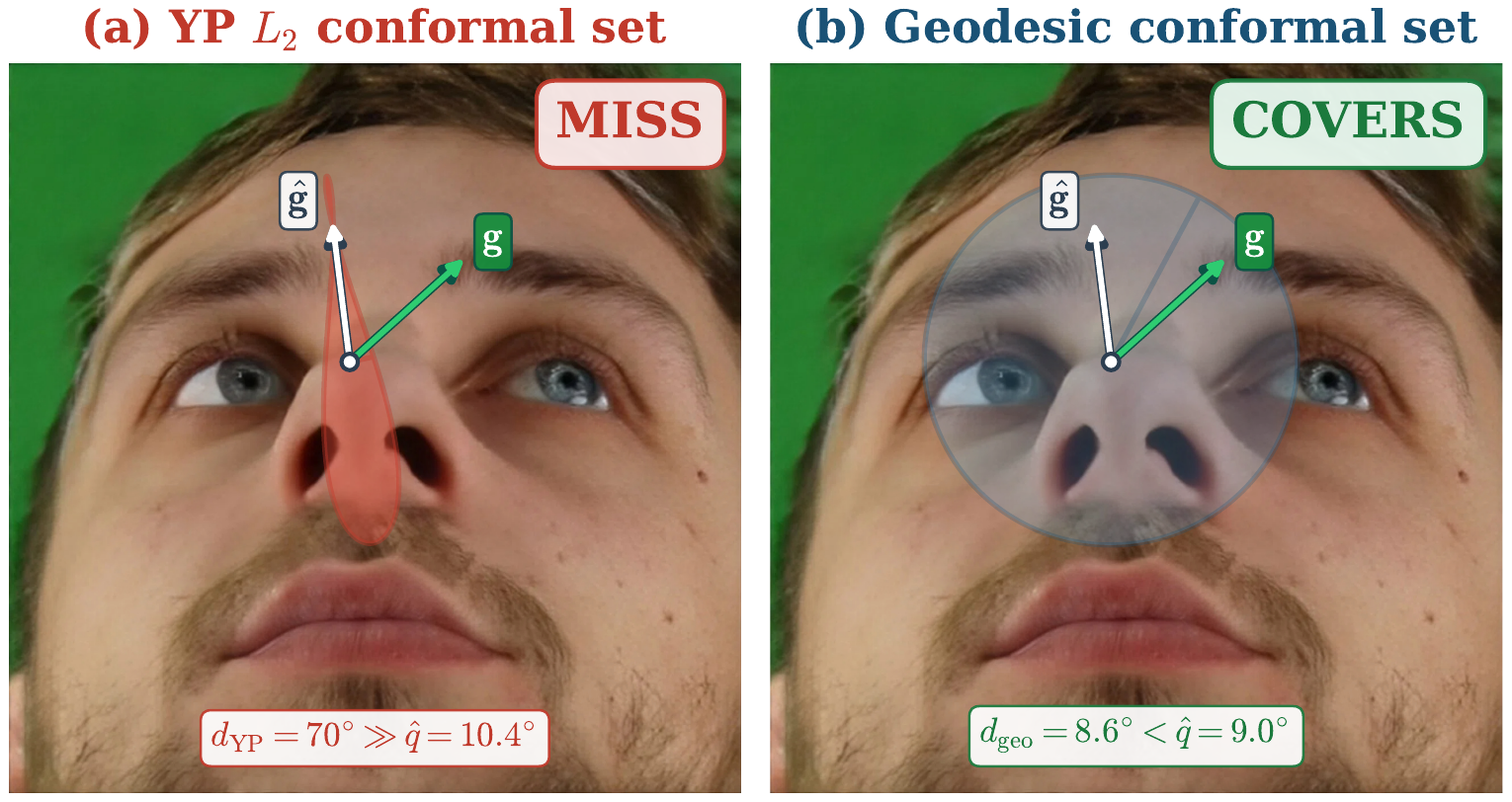}
    \caption{\textbf{Chart distortion breaks conformal coverage}
    (schematic overlay, ETH-XGaze, $|\mathrm{pitch}|=82.5^{\circ}$).
    (a)~The yaw--pitch (YP) $L_2$ conformal set ($\hat{q}=10.4^{\circ}$)
    collapses into a narrow teardrop because
    $\cos(82.5^{\circ})\!\approx\!0.13$ compresses the yaw
    dimension; the ground truth~$\mathbf{g}$ falls outside
    ($d_{\mathrm{YP}}=70^{\circ}\gg\hat{q}$).
    (b)~The geodesic conformal set ($\hat{q}=9.0^{\circ}$)
    forms a uniform circular cap on~$S^2$ and covers~$\mathbf{g}$
    ($d_{\mathrm{geo}}=8.6^{\circ}<\hat{q}$).
    The prediction and calibration data are the same; only the scoring geometry differs.
    Chart coordinates inflate the true angular gap by
    $8.2$ times.}
    \label{fig:gaze_sphere}
\end{figure*}

\begin{enumerate}[leftmargin=*,nosep]
\item[\textbf{C1.}] \textbf{Coverage audit for gaze and head pose.}
Across four datasets and two manifolds ($\Stwo$, $\SOthree$),
we observe drops of 30--50\pp{} in slice-conditional coverage
for yaw--pitch and Euler-angle conformal scores despite
correct marginal calibration.
A controlled experiment on $\SOthree$ with synthetic isotropic
noise isolates coordinate geometry as the sole cause.

\item[\textbf{C2.}] \textbf{Impossibility of scalar chart correction.}
We prove (Proposition~\ref{prop:impossibility}) that scalar
adaptive conformal methods can change the size of prediction
sets but not their local shape.
The axis ratios are determined by the metric tensor of the
chart and therefore cannot be corrected by scalar
thresholding.

\item[\textbf{C3.}] \textbf{A diagnostic and practical remedy.}
We show that the Riemannian volume density provides a simple
diagnostic for where chart-based systems are likely to
undercover.
Using coordinate-free geodesic scores eliminates the
distortion entirely, requiring no retraining and adding
negligible computation
(${\le}0.02\,\mu$s per sample).

\end{enumerate}

\section{Related Work}
\label{sec:related}

\paragraph{Conformal prediction.}
Split conformal prediction~\cite{lei2018distribution,vovk2005algorithmic}
provides distribution-free marginal coverage under exchangeability,
with extensions to adaptive
settings~\cite{papadopoulos2002inductive,romano2019conformalized},
group-conditional guarantees~\cite{vovk2005algorithmic},
and distribution shift~\cite{gibbs2021adaptive,tibshirani2019conformal};
see~\cite{stephen2021gentle,fontana2023conformal,zhou2025conformal}
for surveys.
Barber~\etal~\cite{barber2021limits} show that marginal guarantees
do not imply conditional ones without further assumptions.
Most conformal methods target Euclidean outputs; structured geometric
outputs remain underexplored.
Cholaquidis~\etal~\cite{cholaquidis2025conformal} recently prove
marginal validity for geodesic scores on Riemannian manifolds.
Our work complements theirs. We show that commonly used
chart-based scores in gaze and pose pipelines can produce
systematic slice-conditional undercoverage. We further prove
that scalar threshold adaptation cannot correct this geometric
distortion.

\paragraph{Gaze and head pose estimation.}
Appearance-based gaze
methods~\cite{zhang2020eth,kellnhofer2019gaze360,cheng2022gaze,abdelrahman2023l2cs, lei2023end}
and head pose
estimators~\cite{ruiz2018fine,yang2019fsa,hempel20226d,martyniuk2022dad}
typically report angular or per-component MAE;
uncertainty-aware
variants~\cite{zheng2023confidence,zhong2024uncertainty,cantarini2022hhp, lei2025mac, lei2025quantifying, wang2025ptgaze}
provide confidence estimates but lack distribution-free guarantees.
A persistent gap can arise between representation, training, calibration, and evaluation geometry.
Gaze networks may predict yaw--pitch coordinates before angular-error evaluation~\cite{abdelrahman2023l2cs,zhang2020learning, lei2023end}, while head-pose systems may train with continuous rotation representations but still decode, report, or calibrate in Euler coordinates~\cite{zhou2019continuity,hempel20226d,huynh2009metrics}.
Our analysis concerns the geometry of the conformal score: if the score is computed in yaw--pitch or Euler coordinates, the conformal set inherits chart distortion.
To our knowledge, no existing pipeline audits conformal reliability across pose or gaze ranges.

\paragraph{Uncertainty quantification for orientations.}
Deep ensembles~\cite{lakshminarayanan2017simple}, MC dropout~\cite{gal2016dropout}, heteroscedastic regression~\cite{kendall2017uncertainties}, and distributional approaches on $\SOthree$~\cite{mohlin2020probabilistic,prokudin2018deep} provide uncertainty estimates without finite-sample coverage guarantees.
Johnstone~\etal~\cite{johnstone2021conformal} study learned anisotropic conformal sets in Euclidean spaces; extending such methods to $\Stwo$ or $\SOthree$ would require estimating local geometry.
Directional distributions could give locally geometric sets but sacrifice distribution-free guarantees.

\paragraph{Positioning.}
Geometric distortion in Euler and yaw--pitch charts is well known~\cite{huynh2009metrics,zhou2019continuity}, but its consequences for conformal reliability have not been studied.
To our knowledge, no prior work has shown that a nominally calibrated conformal system can silently undercover by 30--50\pp{} in safety-critical output regions, 
and no prior work has proved that standard adaptive remedies
(normalised CP, CQR) are structurally unable to correct this. Our approach preserves the distribution-free guarantee of conformal prediction while eliminating chart-induced distortion via coordinate-free scoring.
\section{Background and Formal Analysis}
\label{sec:method}

\paragraph{Conformal prediction primer.} Given a predictor $\hat y(x)$ and calibration scores $s_i=s(y_i,\hat y_i)$, split conformal prediction chooses the empirical $(1-\alpha)$ quantile $\hat q_\alpha$ and returns $\mathcal C(x)=\{y:s(y,\hat y(x))\le \hat q_\alpha\}$. Under exchangeability, $\Pr\{Y\in\mathcal C(X)\}\ge 1-\alpha$, but marginal coverage need not hold on slices such as high gaze pitch or near-gimbal-lock poses. We study how the score $s$ affects such slice-conditional coverage. \paragraph{Notation.} Let $\mathcal M$ be the output manifold, $\varphi:\mathcal M\supset U\to D\subset\mathbb R^d$ a chart, $\xi=\varphi(y)$ coordinates, $D\varphi$ the chart differential, $G(\xi)$ the metric tensor, and $\rho(\xi)=\sqrt{\det G(\xi)}$ the volume density.

Figure~\ref{fig:pipeline} summarises the protocol.

\subsection{Setup}
\label{sec:setup}

We analyse the common conformal-wrapper choice in which predictions and labels are represented in a coordinate chart and the nonconformity score is the Euclidean norm of the coordinate residual.
For gaze, this chart score is
$s_{\mathrm{chart}}=\|(\Delta\psi,\Delta\theta)\|_2$ in yaw--pitch coordinates.
For head pose, the analogous score uses ZYX Euler residuals.
This is distinct from the final benchmark metric: a system may report angular error while a downstream conformal wrapper still calibrates coordinate residuals.
(Section~\ref{sec:scores} defines the full score taxonomy;
here we focus on $s_{\mathrm{chart}}$ and the geodesic score
$s_{\mathrm{geo}} = d_g(y, \hat{y})$.)

The discrepancy is captured by the \emph{Riemannian metric
tensor} $G(\xi)$, which describes how infinitesimal coordinate
displacements relate to true distances and volumes on the
manifold~\cite{docarmo1992riemannian,lee2018introduction}:
\begin{equation}
  ds^2 = d\xi^\top G(\xi)\, d\xi, \qquad
  d\mathrm{Vol} = \sqrt{\det G(\xi)}\, d\xi.
  \label{eq:metric}
\end{equation}
When $G = I_d$, coordinates faithfully represent geometry and
$s_{\mathrm{chart}}$ coincides with $s_{\mathrm{geo}}$.
When $G \ne I_d$, the chart score distorts intrinsic distances.

For yaw--pitch on $\Stwo$, the metric is
$G(\psi, \theta) = \mathrm{diag}(\cos^2\!\theta,\; 1)$.
Yaw differences are scaled by $\cos\theta$, which vanishes at
the poles.
For ZYX Euler angles on $\SOthree$ (the convention used
throughout; other sequences have analogous singularities at
different angles), the metric (Supplementary Sec.~A) is
\begin{equation}
  G(\alpha,\beta,\gamma) =
  \begin{pmatrix}
    1 & 0 & -\!\sin\beta \\
    0 & 1 & 0 \\
    -\!\sin\beta & 0 & 1
  \end{pmatrix}\!,
  \label{eq:metric_so3}
\end{equation}
yielding $\det G = \cos^2\!\beta$, which degenerates at gimbal lock
($|\beta| = 90^\circ$).
In both cases, $\lambda_{\min}(G) \to 0$ near the singularity,
so a fixed chart-score radius corresponds to a vanishing
manifold volume.


\paragraph{Consequences for conformal prediction.}
Split conformal prediction~\cite{vovk2005algorithmic} selects
a threshold $\hat{q}_\alpha$ on calibration scores to guarantee
marginal coverage $1-\alpha$, where $\alpha \in (0,1)$ is the
target miscoverage level.
Because most calibration samples lie where $G \approx I_d$, the
threshold is set for those regions.
Near singularities, the same threshold carves out a manifold
region that is too small, producing systematic undercoverage.
All experiments unwrap angles to $[-\pi, \pi]$; the effects
documented below arise from metric distortion, not wrap-around.

\begin{figure}[t]
\centering
\includegraphics[width=\linewidth]{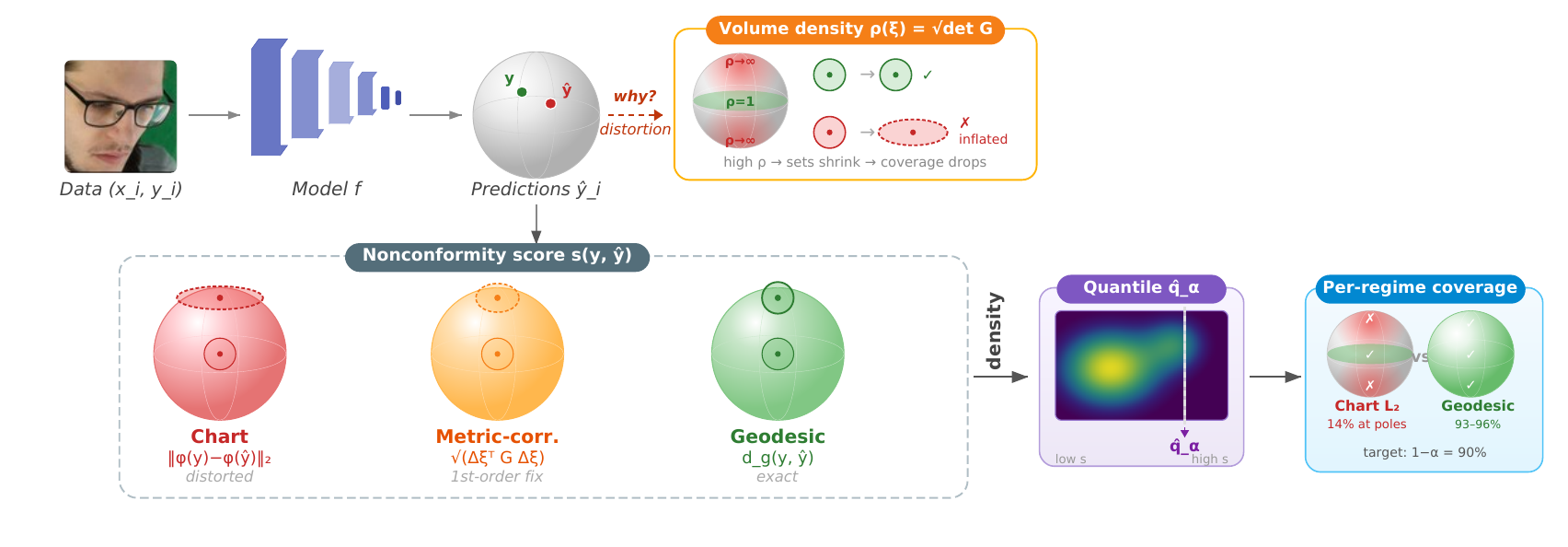}
\caption{\textbf{Protocol overview.}
The key decision point is the nonconformity score.
Chart-based scores inherit coordinate distortion; geodesic scores
are coordinate-free.
The volume density~$\CVF$ indicates where chart scores are likely to fail.}
\label{fig:pipeline}
\end{figure}
\subsection{Coverage Distortion from the Metric Tensor}
\label{sec:coverage_distortion}
We quantify how chart distortion affects the distribution of
residuals and therefore conditional coverage.
Fix a prediction $\hat{Y}(x) = p \in \mathcal{M}$ and model
local error in the tangent space:
$Y = \exp_p(\varepsilon)$,
$\varepsilon \sim \mathcal{N}(0, \Sigma_p)$,
where $\Sigma_p \succ 0$ is an arbitrary positive-definite
covariance and $\varepsilon$ is small enough that $Y$ remains
in the chart domain.
The isotropic special case $\Sigma_p = \sigma^2 I_d$ isolates
chart distortion from model anisotropy; the general case shows
how the two interact.

\begin{proposition}[Chart-induced coverage distortion]
\label{prop:coverage_distortion}
Under the local error model above, expressed in a $g$-orthonormal basis of $T_p\mathcal{M}$, the chart residual
$\Delta\xi_p = \varphi(Y) - \varphi(p)$ satisfies
$\Delta\xi_p \approx \mathcal{N}(0,\; J_p\,\Sigma_p\,J_p^\top)$
where $J_p = D(\varphi \circ \exp_p)|_0$.
Let $R_p = \|\Delta\xi_p\|_2$,
$F_p(r) = \Pr(R_p \le r)$, and
$\mu_1(p) \le \cdots \le \mu_d(p)$ be the eigenvalues
of the combined matrix $M_p = J_p\,\Sigma_p\,J_p^\top$.
Then for all $r \ge 0$,
\begin{equation}
  \Phi_d\!\left(\frac{r}{\sqrt{\mu_d(p)}}\right)
  \le F_p(r) \le
  \Phi_d\!\left(\frac{r}{\sqrt{\mu_1(p)}}\right)\!,
  \label{eq:sandwich}
\end{equation}
where $\Phi_d$ is the CDF of the $\chi_d$ distribution.
In the isotropic case ($\Sigma_p = \sigma^2 I_d$), letting $\xi_p = \varphi(p)$, the combined matrix is $M_p = \sigma^2 G(\xi_p)^{-1}$ and the set of eigenvalues reduces to $\{\sigma^2 / \lambda_i^G\}_{i=1}^d$, recovering a bound in terms of the metric-tensor eigenvalues alone.
\end{proposition}

\begin{proof}[Proof sketch]
Linearising $\varphi \circ \exp_p$ gives
$\Delta\xi_p \approx \mathcal{N}(0,\; J_p \Sigma_p J_p^\top)$.
The squared chart radius $R_p^2 = \|\Delta\xi_p\|^2$ is a
quadratic form in a Gaussian vector; the Rayleigh quotient
bounds $\mu_1 \|w\|^2 \le w^\top M_p\, w \le \mu_d \|w\|^2$
(with $w \sim \mathcal{N}(0, I_d)$)
yield~\eqref{eq:sandwich}.
When $\Sigma_p = \sigma^2 I_d$, since $D\exp_p|_0 = I$, in a $g$-orthonormal 
basis of $T_p\mathcal{M}$ we have $J_p = D\varphi_p$. Thus, metric 
compatibility yields $J_p J_p^\top = G^{-1}$, so $M_p = \sigma^2 G^{-1}$, 
and the bounds depend only on the metric tensor.
Full derivation in Supplementary Sec.~A.
\end{proof}

In the isotropic case, where $\lambda_{\min}^G(p)$ is
small, near poles on $\Stwo$
($\lambda_{\min}^G = \cos^2\!\theta$) or gimbal lock on $\SOthree$
($\lambda_{\min}^G = 1 - |\sin\beta|$), chart residuals become inflated in compressed directions.
Thus a fixed chart-score threshold accepts a vanishing intrinsic volume near the singularity, producing undercoverage.
When model error is anisotropic ($\Sigma_p \ne \sigma^2 I_d$),
the coverage distortion depends on the alignment between
$\Sigma_p$ and the metric eigenvectors.
If the model anisotropy opposes the chart distortion, the
effect can be attenuated.
If it is aligned with the distortion, the effect is amplified.
Supplementary Sec.~A.1 confirms both cases empirically.
Setting $r = \hat{q}_\alpha$ in~\eqref{eq:sandwich} gives
explicit conditional coverage bounds
(Supplementary Sec.~A, Corollary~1).
At singularities where $\lambda_{\min}^G \to 0$ and model error
remains bounded, the lower bound tends to zero.
Coverage is \emph{redistributed}, not
destroyed~\cite{vovk2005algorithmic}.
Regions near singularities undercover, while well-conditioned
regions overcover.
The residual ellipse eccentricity reaches $5.76{:}1$ at
$|\theta| = 80^\circ$ on $\Stwo$ and $4.35{:}1$ at
$|\beta| = 60^\circ$ on $\SOthree$; both ETH-XGaze and BIWI
contain substantial test samples at these angles.
The propositions provide local, first-order predictions;
the experiments in Sec.~\ref{sec:experiments} confirm these
predictions hold at finite radius across all settings tested.

\paragraph{Validity of the linearisation.}
Both propositions rely on a first-order approximation of the coordinate chart $\varphi$.
Near singularities, the metric tensor varies rapidly.
In these regions the approximation becomes less tight.
The quantitative bounds may therefore be loose when a prediction
set spans a region over which $G$ changes substantially.
Our experiments (Tables~\ref{tab:so3_biwi}--\ref{tab:aflw2000})
confirm that the qualitative predictions hold well beyond the
linearisation regime.
These include directional collapse, redistribution, and the
impossibility of scalar correction.
However, the numerical bounds of
Proposition~\ref{prop:coverage_distortion} should be interpreted
as indicative rather than exact at extreme angles.


\paragraph{Diagnostic: the volume density.}
The Riemannian volume density
$\sqrt{\det G(\xi)}$ denoted as $\CVF(\xi)$
measures how much manifold volume corresponds to one unit of
coordinate volume.
For both standard vision charts, $\CVF$ and $\lambda_{\min}$
are monotonically related
($\CVF = |\cos\theta|$ on $\Stwo$;
$|\cos\beta|$ on $\SOthree$),
making $\CVF$ an easily computed, closed-form proxy that
strongly correlates with where coverage collapse is observed
($r > 0.90$ on controlled settings, $r = 0.81$ on AFLW2000;
Sec.~\ref{sec:experiments}).
For these standard charts, $\CVF$ reduces to $|\cos(\cdot)|$
and serves primarily as an interpretability aid that maps
known coordinate geometry to expected coverage behaviour.
On general manifolds with non-diagonal or higher-dimensional
metric tensors, the eigenvalue-based bounds of
Proposition~\ref{prop:coverage_distortion} provide a more
informative diagnostic (Supplementary Sec.~C).

\subsection{Why Scalar Adaptive Methods Cannot Fix This}
\label{sec:impossibility}

Normalised CP~\cite{papadopoulos2002inductive} and
CQR~\cite{romano2019conformalized} rescale the acceptance
threshold based on local difficulty.
We show that this broad class of scalar threshold-modulating
methods is structurally unable to correct
chart distortion: it is a \emph{shape} problem, not a
\emph{scale} problem.

\begin{proposition}[Local impossibility of scalar chart correction]
\label{prop:impossibility}
Let $\varphi$ be a coordinate chart on a $d$-dimensional
Riemannian manifold $(\mathcal{M}, g)$, and assume the chart is nonsingular at the prediction $\hat{y}(x)$. Consider any
conformal predictor whose acceptance sets take the form
\[
  \mathcal{C}(x) = \bigl\{y \in \mathcal{M} :
    h\bigl(\|\varphi(y) - \varphi(\hat{y}(x))\|_2,\; x\bigr)
    \le q\bigr\},
\]
where $h(\cdot, x)$ is nondecreasing for each~$x$.
Then $\mathcal{C}(x)$ is the preimage of a chart-space ball of
radius $r(x)$. In a $g$-orthonormal basis of the tangent space at
$\hat{p} = \hat{y}(x)$, letting $J = D(\varphi^{-1})_{\hat\xi}$, the 
tangent-space approximation of $\mathcal{C}(x)$ is the ellipsoid
$\{\delta \in T_{\hat{p}}\mathcal{M} : \delta^\top (JJ^\top)^{-1} \delta \le r(x)^2\}$
with semi-axes $r(x)\sqrt{\lambda_i(\hat\xi)}$, where $\lambda_i(\hat\xi)$ 
are the eigenvalues of the metric tensor $G(\hat\xi) = J^\top J$.
The local axis ratio
\begin{equation}
  \frac{\mathrm{major}}{\mathrm{minor}}
  = \sqrt{\frac{\lambda_{\max}(\hat\xi)}
               {\lambda_{\min}(\hat\xi)}}
  \label{eq:axis_ratio}
\end{equation}
is independent of both $h$ and $r(x)$.
If $G(\hat\xi) \ne c\,I_d$ for any $c > 0$, no choice of $h$
can make the acceptance set locally isotropic.
\end{proposition}

\begin{proof}[Proof sketch]
Monotonicity gives
$\{y : h(s,x) \le q\} = \{y : s \le r(x)\}$, a chart ball.
Let $J = D(\varphi^{-1})_{\hat\xi}$. Expressing $J$ in a $g$-orthonormal 
frame of $T_{\hat p}\mathcal{M}$, the pullback metric in chart coordinates 
is $G = J^\top J$.
The chart constraint $\|\Delta\xi\|_2 \le r$ maps to the tangent space 
via $\delta = J\Delta\xi$, yielding $\delta^\top (J J^\top)^{-1} \delta \le r^2$. 
Because the non-zero eigenvalues of $J J^\top$ are identical to those of $J^\top J = G$, 
the ellipsoid's semi-axes are exactly $r\sqrt{\lambda_i(G)}$. 
Full derivation in Supplementary Sec.~A.
\end{proof}
\noindent\textbf{Instantiation.}
On $\Stwo$, \eqref{eq:axis_ratio} gives
$1/|\cos\theta|$ (diverges at poles).
\newline On $\SOthree$,
$\sqrt{(1{+}|\sin\beta|)/(1{-}|\sin\beta|)}$ (diverges at
gimbal lock).

\paragraph{Scope of the impossibility result.}
Proposition~\ref{prop:impossibility} applies to conformal predictors
whose acceptance regions are radial in chart coordinates,
\ie sublevel sets of a scalar function of
$\|\varphi(y) - \varphi(\hat{y}(x))\|_2$.
Methods that explicitly learn anisotropic or non-radial
acceptance shapes (e.g., full covariance or metric learning in
the tangent space) fall outside this class.
Our result therefore does not preclude shape correction via
non-radial scores; it shows that scalar rescaling of a chart
norm, including normalised conformal prediction and
conformalized quantile regression, cannot eliminate
chart-induced anisotropy.
\paragraph{Scale vs.\ shape correction.}
We distinguish two types of correction, formalised in the
three-layer decomposition of Sec.~\ref{sec:three_layer}.
\emph{Scale correction} (normalised CP, CQR) adjusts the
radius $r(x)$ without changing axis ratios; it addresses
heteroscedastic model error but not chart distortion.
\emph{Shape correction} changes the score geometry.
The metric-corrected score
$\sqrt{\Delta\xi^\top G\,\Delta\xi}$ (first-order) and the
geodesic score $d_g(y,\hat{y})$ (exact) produce locally
isotropic acceptance sets (Layer~3).
Combining shape and scale correction substantially reduces conditional gaps
across experiments (Sec.~\ref{sec:experiments}).
Methods that learn anisotropic prediction
sets~\cite{johnstone2021conformal} could achieve shape
correction but do not yet exist for $\Stwo$ or $\SOthree$;
closed-form geodesics make this unnecessary on these manifolds.


\subsection{Scoring Functions}
\label{sec:scores}

Table~\ref{tab:score_taxonomy} classifies scores by geometric
properties and Layer-3 correction type.

\begin{table}[t]
\centering
\caption{Nonconformity score taxonomy.
$^*$Monotone-equivalent to geodesic; identical conformal sets.}
\label{tab:score_taxonomy}
\small
\begin{tabular}{@{}llccc@{}}
\toprule
Score & Manifold & Chart? & Intrinsic? & Shape \\
\midrule
Yaw/pitch $L_2$ & $\Stwo$ & \textbf{Yes} & No & None \\
Euler ZYX $L_2$ & $\SOthree$ & \textbf{Yes} & No & None \\
\midrule
Metric-corr.\ YP & $\Stwo$ & Partial & No & 1st order \\
Metric-corr.\ Euler & $\SOthree$ & Partial & No & 1st order \\
\midrule
Geodesic (arc) & $\Stwo$ & No & \textbf{Yes} & Exact \\
$\R^3$ chord & $\Stwo$ & No & \textbf{Yes}$^*$ & Exact \\
Geodesic (angle) & $\SOthree$ & No & \textbf{Yes} & Exact \\
Quaternion $d_q$ & $\SOthree$ & No & \textbf{Yes}$^*$ & Exact \\
\bottomrule
\end{tabular}
\end{table}

\paragraph{Geodesic scores.}
On $\Stwo$: $d_{\Stwo}(v,\hat{v}) = \arccos(v^\top \hat{v})$.
On $\SOthree$:
$d_{\SOthree}(R_1, R_2)
= \|\log(R_1^\top R_2)\|_F / \sqrt{2}$.
Because conformal prediction depends only on score
\emph{ranking}, any strictly monotone surrogate of the geodesic
distance yields identical conformal sets. For example, on $\Stwo$
one may use $1 - v^\top \hat{v}$ (or chord distance), and on
$\SOthree$ $1 - |q^\top \hat{q}|$, thereby avoiding
trigonometric functions and matrix logarithms.


\paragraph{Metric-corrected scores.}
A first-order correction weights chart residuals by the local
metric:
\begin{equation}
  s_{\mathrm{mc}}
  = \sqrt{\Delta\xi^\top G(\bar\xi)\, \Delta\xi}
  = \sqrt{(\Delta\psi)^2 \cos^2\!\bar\theta
    + (\Delta\theta)^2},
  \label{eq:cw}
\end{equation}
evaluated at the midpoint $\bar\xi$.
This removes first-order anisotropy but retains higher-order
distortion.


\paragraph{Normalised scores.}
Dividing any base score by a difficulty estimate
$\hat\sigma(x)$~\cite{papadopoulos2002inductive} addresses
heteroscedastic model error (Layer~2).
In all experiments, $\hat\sigma(x)$ is a $k$-NN mean absolute
residual in feature space (details in Supplementary Sec.~F).
By Proposition~\ref{prop:impossibility}, normalisation cannot
alter acceptance-set \emph{shape}: normalised chart scores
remain anisotropic; normalised geodesic scores address both
layers simultaneously.

\subsection{Three-Layer Decomposition}
\label{sec:three_layer}

We decompose conditional coverage gaps into three layers:
marginal validity (Layer~1, by construction),
heteroscedastic model error (Layer~2, scale correction), and
chart-induced distortion
(Layer~3, shape correction only;
Proposition~\ref{prop:impossibility}).
If scalar methods could absorb chart distortion, normalised
yaw--pitch would match normalised geodesic; a persistent
40\pp{} gap on ETH-XGaze confirms they cannot
(Sec.~\ref{sec:experiments}).
Thus, plain geodesic scoring removes Layer~3 chart distortion, but residual scale variation may remain; normalised geodesic combines Layer~3 shape correction with Layer~2 scale correction. On ETH-XGaze, plain geodesic scoring raises near-pole coverage from 38.9\% to 74.6\% (Layer~3 removed), while normalised geodesic further closes the remaining gap to 89.3\% (Layer~2 also addressed).
A practical protocol is given in Sec.~\ref{sec:discussion}.

\section{Experiments}
\label{sec:experiments}

We evaluate conformal coverage across two manifolds ($\Stwo$,
$\SOthree$), four datasets, and both controlled and real-model
settings.
All experiments use split conformal
prediction~\cite{vovk2005algorithmic} with $\alpha = 0.10$
(target: 90\% marginal coverage).
Per-regime coverage is reported as mean $\pm$ standard deviation
over multiple random calibration/test splits (subject-disjoint).
Throughout, ``conditional coverage'' means \emph{slice-conditional} coverage: empirical coverage in a specified output region (\eg, $|\mathrm{pitch}|>70^\circ$). This is weaker than pointwise conditional coverage~\cite{barber2021limits} but sufficient to reveal failures hidden by marginal metrics.


\subsection{Controlled $\SOthree$ Setting: Isolating the Mechanism}
\label{sec:exp_biwi}

\paragraph{Setup.}
We use ground-truth rotations from BIWI~\cite{biwi}
($N = 15{,}678$, 24 subjects) and add calibrated isotropic
Lie-algebra noise at $\sigma = 5^\circ$, producing geodesic
MAE of $8.0^\circ$ (comparable to modern
methods~\cite{hempel20226d,yang2019fsa}).
Because the noise is isotropic on the manifold by construction,
any coverage variation is attributable solely to score geometry.

\paragraph{Results.}
Table~\ref{tab:so3_biwi}: all methods achieve $\approx$90\%
marginal coverage.
Near gimbal lock ($|\beta| > 60^\circ$, 4.5\% of samples),
Euler $L_2$ collapses to $57.5 \pm 2.1$\% while geodesic
maintains $89.5 \pm 1.4$\%.
At $|\beta| < 30^\circ$, Euler over-covers at $95.6$\%,
compensating as
Proposition~\ref{prop:coverage_distortion} predicts.

\begin{table}[t]
\centering
\caption{\textbf{Controlled $\SOthree$ experiment} (BIWI,
$\alpha = 0.10$, $\sigma = 5^\circ$, 50 splits).
Near gimbal lock, Euler scoring collapses; geodesic maintains
target coverage. All values in \%. Bold marks the largest absolute deviation from the 90\% target.}
\label{tab:so3_biwi}
\small
\setlength{\tabcolsep}{4pt}
\renewcommand{\arraystretch}{1.05}
\begin{tabular}{@{}lccccc@{}}
\toprule
Region & Geo.$^\dagger$ & Euler $L_2$ & Met.-corr. & N.\ Euler & N.\ Geo. \\
\midrule
Overall            & $90.0 {\pm} 0.5$ & $89.9 {\pm} 0.5$ & $89.9 {\pm} 0.5$ & $90.0 {\pm} 0.4$ & $90.0 {\pm} 0.5$ \\
$|\beta|{<}30^\circ$ & $89.7 {\pm} 0.6$ & $95.6 {\pm} 0.4$ & $90.2 {\pm} 0.6$ & $92.9 {\pm} 0.4$ & $89.5 {\pm} 0.6$ \\
$|\beta|{>}60^\circ$ & $89.5 {\pm} 1.4$ & $\mathbf{57.5 {\pm} 2.1}$ & $87.8 {\pm} 1.5$ & $72.7 {\pm} 1.7$ & $89.0 {\pm} 1.2$ \\
\midrule
\multicolumn{6}{@{}l}{\small $\CVF$--coverage correlation:
$r = 0.987$\; ($p = 5 \times 10^{-6}$).} \\
\bottomrule
\multicolumn{6}{@{}l}{\scriptsize $^\dagger$Quaternion distance
produces identical results (monotone equivalence).} \\
\end{tabular}
\end{table}

\paragraph{Robustness.}
The pattern is robust across noise magnitudes
($\sigma \in \{3^\circ, 8^\circ, 12^\circ\}$; Euler stays at
${\sim}61$\%, geodesic at ${\approx}91$\%) and under anisotropic
noise (ratios up to 4:1, CVF $r > 0.96$).
Only at extreme anisotropy (5.7:1:1.3) does Euler's failure
attenuate (84.3\%); geodesic remains at 90\%.
Full sweep in Supplementary Sec.~A.1.


\subsection{$\Stwo$ Gaze Estimation: Real Models}
\label{sec:exp_gaze}

\paragraph{Setup.}
We train ResNet-18 gaze estimators on
ETH-XGaze ($N_\text{test} = 117{,}360$)
and Gaze360~\cite{kellnhofer2019gaze360}
($N_\text{test} = 25{,}969$), comparing the non-redundant scoring functions in
(Table~\ref{tab:score_taxonomy}).
We define ``near-pole'' as $|\mathrm{pitch}| > 70^\circ$,
within the physiological range of combined eye and head
rotation (Sec.~\ref{sec:intro}) and relevant to
safety-critical gaze targets including the instrument
cluster, lap-mounted devices, and side
mirrors~\cite{nikan2022appearance}.

\paragraph{Results.}
Marginal coverage is indistinguishable across methods ($90.8$--$91.0$\%, $\le9.18^\circ$ geodesic radius, $\le0.097$~sr solid angle). Table~\ref{tab:gaze_results} shows a sharp near-pole disparity: YP $L_2$ achieves $38.9\pm3.9$\%, while coordinate-free scoring reaches $74.6\pm3.9$\%. Table~\ref{tab:gaze360_main} replicates this on Gaze360: $42.0$\% for YP $L_2$, $77.2$\% for geodesic, and $86.5$\% for normalised geodesic. 

\paragraph{Scale vs.\ shape.} Normalised YP improves near-pole coverage from 38.9\% to 50.4\%, but remains 40\pp{} below target; metric-corrected scoring reaches 74.1\%, and normalised geodesic reaches 89.3\%, confirming Proposition~\ref{prop:impossibility}.


\begin{table}[t]
\centering
\caption{\textbf{$\Stwo$ gaze estimation} (ETH-XGaze R18,
$\alpha = 0.10$, 5 splits).
Conditional coverage by output regime. Wilson 95\% CIs in brackets. Bold marks the largest absolute deviation from the 90\% target.}
\label{tab:gaze_results}
\small
\setlength{\tabcolsep}{4pt}
\renewcommand{\arraystretch}{1.05}
\begin{tabular}{@{}lccc@{}}
\toprule
Method & Overall & Equator & Near Pole \\
\midrule
Geodesic$^\dagger$ & $91.0 \pm 0.7$ & $94.0 \pm 1.1$ & $74.6 \pm 3.9$\;\scriptsize[72.3, 76.8] \\
Yaw/Pitch $L_2$        & $90.8 \pm 0.7$ & $96.5 \pm 1.0$ & $\mathbf{38.9 \pm 3.9}$\;\scriptsize[36.4, 41.5] \\
Metric-corrected       & $91.0 \pm 0.7$ & $94.1 \pm 1.1$ & $74.1 \pm 3.9$\;\scriptsize[71.7, 76.3] \\
Normalised YP          & $89.0 \pm 3.3$ & $94.2 \pm 2.1$ & $50.4 \pm 5.9$\;\scriptsize[47.8, 53.0] \\
Norm.\ geodesic        & $89.1 \pm 3.6$ & $90.6 \pm 3.0$ & $89.3 \pm 1.9$\;\scriptsize[87.6, 90.8] \\
\midrule
\multicolumn{4}{@{}l}{\small $\CVF$--coverage:
$r = 0.957$ (ETH-XGaze), $r = 0.976$ (Gaze360).} \\
\bottomrule
\multicolumn{4}{@{}l}{\scriptsize $^\dagger$Geodesic and $\R^3$ chord
produce identical acceptance regions.} \\
\end{tabular}
\end{table}

\begin{table}[t]
\centering
\caption{\textbf{Gaze360 conditional coverage} (R18, $\alpha=0.10$, 5 splits; means shown, full mean$\pm$std in Supplementary Sec.~B).
Bold marks the largest absolute deviation from the 90\% target.}
\label{tab:gaze360_main}
\small
\setlength{\tabcolsep}{4pt}
\begin{tabular}{@{}lccc@{}}
\toprule
Score & Overall & Equator & Near Pole \\
\midrule
Geodesic & $89.7$ & $90.6$ & $77.2$ \\
YP $L_2$ & $89.9$ & $92.2$ & $\mathbf{42.0}$ \\
Metric-corr. & $89.7$ & $90.7$ & $74.5$ \\
Norm. YP & $89.9$ & $91.6$ & $50.8$ \\
Norm. geodesic & $89.8$ & $90.1$ & $86.5$ \\
\bottomrule
\end{tabular}
\end{table}


\subsection{Real-Model $\SOthree$ Head Pose: AFLW2000-3D}
\label{sec:exp_aflw}

\paragraph{Setup.}
We run SixDRepNet~\cite{hempel20226d} (RepVGG-B1g2) on
AFLW2000-3D~\cite{aflw2000} ($N = 1{,}969$), which includes
substantial profile views ($|\beta| > 60^\circ$: $n = 278$,
14.1\%).
Normal cervical rotation reaches $\pm70$--$80^\circ$
($144^\circ$ total~\cite{feipel1999normal}), placing
profile-range orientations well within everyday head motion
during conversation, shoulder checks, and blind-spot glances.
SixDRepNet uses the continuous 6D rotation
representation~\cite{zhou2019continuity} internally; the
failure arises \emph{solely} from reverting to Euler
coordinates at scoring time.
No faces or poses are filtered; all 1,969 samples are used
regardless of detection confidence or pose extremity.

\paragraph{Results.}
Table~\ref{tab:aflw2000}: at $|\beta| > 60^\circ$, geodesic
scoring reaches $75.1 \pm 3.7$\% (model weakness, Layer~2);
Euler drops further to $55.2 \pm 4.2$\%, the 20\pp{} gap
isolates chart distortion (Layer~3).
Normalised geodesic reaches $80.3 \pm 3.1$\%, recovering part
of the Layer-2 deficit; normalised Euler remains at
$62.8 \pm 4.0$\%.
Mondrian CP partitions samples into output regimes and calibrates a separate conformal threshold per bin. With a 3-bin geodesic base it restores $90.1 \pm 3.3$\%, showing that group-wise calibration can address Layer~2/model-regime effects once the base score geometry is intrinsic.

\begin{table}[t]
\centering
\caption{\textbf{Real-model conformal on AFLW2000-3D}
(SixDRepNet, $\alpha = 0.10$, 50 splits).
The 20\pp{} gap at $|\beta|{>}60^\circ$ isolates Layer~3. All values in \%. Bold marks the largest absolute deviation from the 90\% target.}
\label{tab:aflw2000}
\small
\setlength{\tabcolsep}{4pt}
\renewcommand{\arraystretch}{1.05}
\begin{tabular}{@{}lccccc@{}}
\toprule
Region & Geo.$^\dagger$ & Euler $L_2$ & Met.-corr. & N.\ Geo. & Mondrian \\
\midrule
Overall & $89.8{\pm}1.4$ & $89.8{\pm}1.2$ & $89.8{\pm}1.3$ & $89.5{\pm}1.5$ & --- \\
$|\beta|{<}30$ & $95.6{\pm}0.8$ & $98.2{\pm}0.5$ & $95.0{\pm}0.9$ & $93.1{\pm}1.0$ & --- \\
$|\beta|{>}60$ & $75.1{\pm}3.7$ & $\mathbf{55.2{\pm}4.2}$ & $78.1{\pm}2.9$ & $80.3{\pm}3.1$ & $90.1{\pm}3.3$ \\
\midrule
\multicolumn{6}{@{}l}{\small $\CVF$--coverage: $r = 0.807$\;
($p = 1.5 \times 10^{-2}$).} \\
\bottomrule
\multicolumn{6}{@{}l}{\scriptsize $^\dagger$Quaternion distance
produces identical results.}
\end{tabular}
\end{table}


\subsection{Separating the Layers: Cross-Setting Synthesis}
\label{sec:three_layer_exp}

The three-layer decomposition
(Sec.~\ref{sec:three_layer}) predicts that normalised chart
scores should improve over unnormalised (Layer~2 addressed)
but fall short of normalised geodesic (Layer~3 unresolved).
This prediction holds across all datasets and backbone settings tested.


\paragraph{Layer~2 gains.}
Normalisation yields consistent but limited improvements:
+15\pp{} on BIWI ($57.5 \to 72.7$\%), +12\pp{} on ETH-XGaze
($38.9 \to 50.4$\%), +8\pp{} on AFLW2000
($55.2 \to 62.8$\%).

\paragraph{Layer~3 residuals.}
In every setting, a large gap persists between normalised
chart and normalised geodesic scores.
39\pp{} on ETH-XGaze ($50.4$ vs.\ $89.3$\%),
16\pp{} on BIWI ($72.7$ vs.\ $89.0$\%),
18\pp{} on AFLW2000 ($62.8$ vs.\ $80.3$\%),
and 43\pp{} on GazeTR-ViT ($38.7$ vs.\ $81.3$\%;
trained on ETH-XGaze, $N_\text{test} = 117{,}360$).
This residual is chart distortion in isolation
(Proposition~\ref{prop:impossibility}).
We do not include a separate CQR baseline because CQR with
chart-norm scores produces radial acceptance sets of the same
form as normalised CP; Proposition~\ref{prop:impossibility}
applies identically, and normalised YP already demonstrates
the scalar-method ceiling.
Notably, the strongest backbone (GazeTR-ViT) produces the
\emph{largest} gap. Better models concentrate predictions in
well-conditioned regions, amplifying the relative penalty at
singularities (Sec.~\ref{sec:ablations}).\\
Fig.~\ref{fig:coverage_ribbon} 
tracks each scoring function’s conditional coverage
from equator to pole across all four datasets.
Chart-based scores (red) collapse sharply beyond moderate angles,
while intrinsic scores (blue) substantially reduce the coverage gap.
\begin{figure*}[t]
    \centering
    \includegraphics[width=0.95\textwidth]{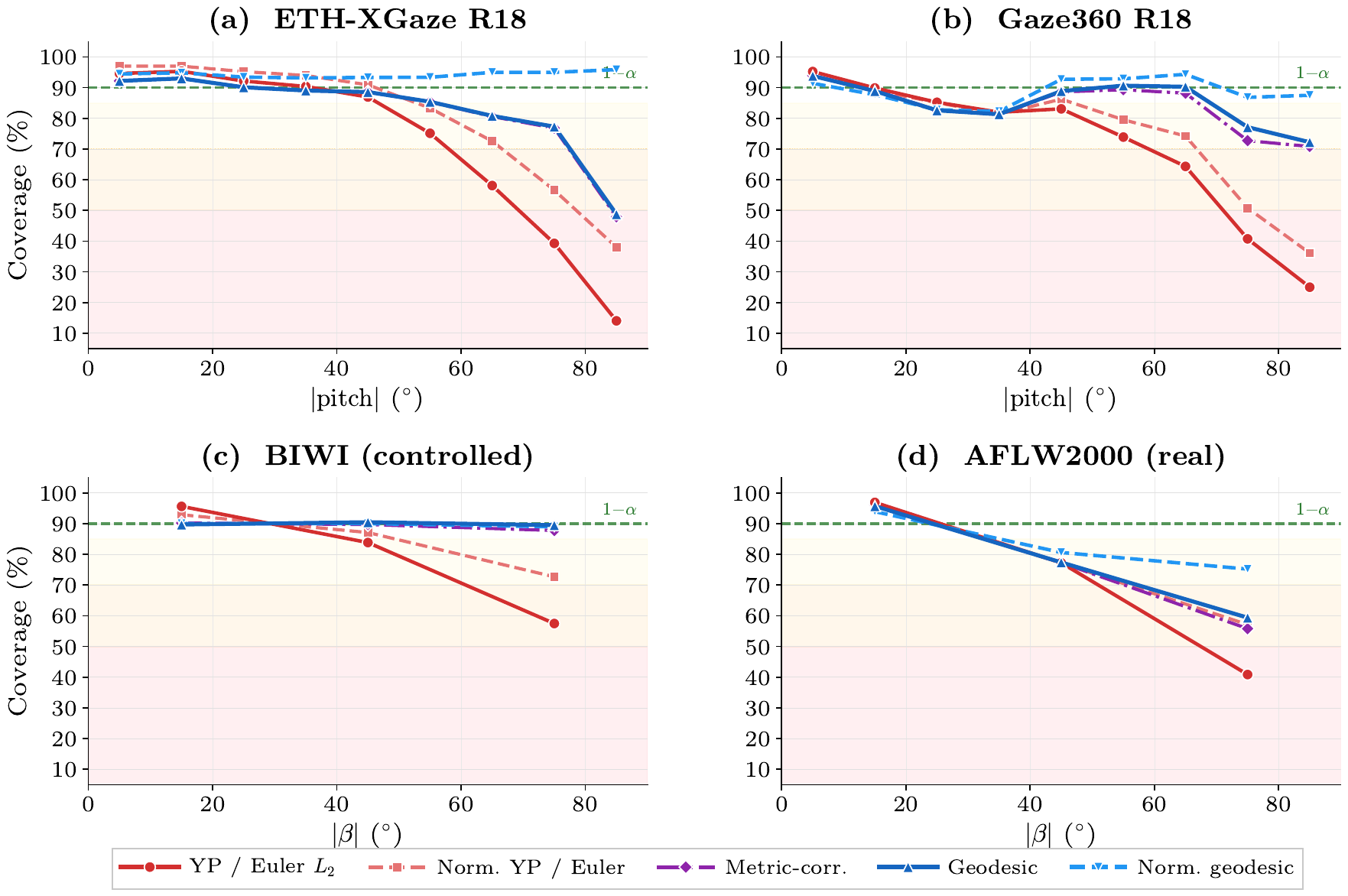}
    \caption{\textbf{Conditional coverage from equator to pole}
    ($\alpha{=}0.10$).
    Shaded bands mark degraded ($<85\%$), severe ($<70\%$),
    and collapsed ($<50\%$) regimes.
    Chart-based scores (red) drop to $14$--$25\%$ at extreme
    angles; normalised geodesic (blue dashed) remains closest to the target and substantially reduces the collapse.
    BIWI~(c) isolates geometry as the sole factor;
    AFLW2000~(d) compounds model error, but intrinsic score-only methods
    retain the smallest gap to target among non-Mondrian scores.}
    \label{fig:coverage_ribbon}
\end{figure*}

\section{Ablations}
\label{sec:ablations}

We summarise key ablations. Full tables appear in Supplementary
Secs.~B and~F.

\paragraph{Stronger models amplify chart distortion.}
With a ViT-S backbone (GazeTR~\cite{cheng2022gaze}, ETH-XGaze),
YP near-pole coverage drops to $30.3 \pm 7.8$\%, the lowest of all
backbones tested. Normalised geodesic reaches
$81.3 \pm 4.6$\%.
Stronger models concentrate predictions in well-conditioned
regions, so the conformal quantile is set more tightly there.
This increases the relative penalty at singularities.

\paragraph{Chart-repair baselines.}
A multi-chart atlas reaches 59.7\% near-pole on Gaze360 versus
77.2\% for geodesic (Supplementary Sec.~F).
A learned Mahalanobis score partially compensates.
It achieves 80.3$\pm$7.0\% on ETH-XGaze and
77.8$\pm$6.2\% on Gaze360 but with higher variance.
On GazeTR it reaches 76.1$\pm$11.1\% versus
81.3$\pm$4.6\% for normalised geodesic.
Normalised geodesic requires no learned parameters and is up to
9\pp{} higher than the best Mahalanobis setting.
Neither approach removes the geometric mismatch predicted by
Proposition~\ref{prop:impossibility}.

\paragraph{Robustness and Runtime.}
The effect persists across calibration sizes with
$<0.5$\pp{} variation from 1\% to 100\% of data.
It persists across binning strategies.
In the ETH-XGaze binning ablation, equal-mass Mondrian bins still yield 66.3\% near-pole coverage, below the corresponding normalised-geodesic result; full details are in Supplementary Sec.~F.
It persists under cross-dataset recalibration.
ETH-XGaze$\to$Gaze360 still gives YP near 54\%.
Geodesic arc, $\R^3$ chord, and quaternion distance produce
identical acceptance regions. Monotone-equivalent geodesic surrogates require only a dot
product and cost ${\le}0.02\,\mu$s per sample on CPU.
Chart-coordinate norms cost about the same.
Metric-corrected scores are slightly slower.
All scoring costs are ${\ll}1$\% of backbone inference time.

\section{Discussion}
\label{sec:discussion}

\paragraph{A hidden failure mode under global summaries.}
The slice-conditional coverage collapse shown in
Sections~\ref{sec:experiments}--\ref{sec:ablations} is largely
invisible under global summary metrics.
Marginal coverage, MAE, and mean set size can appear normal
(Table~\ref{tab:gaze_results}) even when reliability degrades in
specific output regions.
This is concerning in safety-relevant regimes such as extreme pitch
and profile poses~\cite{nikan2022appearance,martyniuk2022dad,patney2016towards}.
Chart coordinates persist because datasets, model heads, and downstream systems often store, decode, or calibrate yaw--pitch/Euler parameters, even when final point-estimation accuracy is reported using angular or geodesic metrics.
A practical fix is to change only the nonconformity score.
For manifold-valued outputs: convert predictions and labels to an intrinsic representation; compute a geodesic or monotone-equivalent score; optionally normalise for heteroscedasticity; calibrate with split conformal prediction; and report slice-conditional coverage alongside marginal coverage.

\paragraph{Scale adaptation cannot correct geometric distortion.}
The results distinguish scale and shape correction.
Scalar adaptive methods adjust set size and can address
heteroscedastic error.
They do not change the local geometry induced by a chart.
Proposition~\ref{prop:impossibility} formalises this limitation,
and the experiments reflect it through consistent gaps between
normalised chart scores and normalised geodesic scores.
This complements known limits on conditional
coverage~\cite{barber2021limits}.
Distribution-free pointwise guarantees are unattainable in general,
but part of the conditional gap observed in geometric settings can
be reduced by using intrinsic scores.

\paragraph{Implications beyond gaze and head pose.}
The choice of nonconformity score is important when outputs lie on a
manifold.
More generally, pipelines that
(i) produce manifold-valued outputs,
(ii) measure residuals in chart coordinates, and
(iii) report only marginal conformal coverage
may undercover in geometrically ill-conditioned regions.
The same diagnostic applies to other geometric tasks, including camera pose in $SE(3)$, angular prediction on products of circles, surface normal estimation on $\Stwo$, and articulated human pose represented as products of joint rotations or directions.
For $SE(3)$, translation is Euclidean while rotation inherits the $\SOthree$ analysis, with the translation--rotation scale chosen by the application.
In such settings, chart singularities or poorly conditioned metric tensors indicate where reliability may degrade.
\section{Limitations}
\label{sec:limitations}

Our approach inherits split conformal prediction's standard
assumptions (exchangeability, held-out calibration data).
Propositions~\ref{prop:coverage_distortion}
and~\ref{prop:impossibility} are local, first-order results.
The bounds may be loose when errors are large or strongly
anisotropic, though robustness experiments
(Sec.~\ref{sec:exp_biwi}) confirm the qualitative predictions
hold under moderate violations (anisotropy up to 4:1 on
$\SOthree$, real model errors on AFLW2000).
We evaluate on $\Stwo$ and $\SOthree$; the theory applies to any Riemannian manifold with a non-isometric or nonconformal chart, but empirical validation on other spaces (\eg $SE(3)$ or articulated pose) is future work.
On manifolds without closed-form geodesics, intrinsic scoring may require numerical approximations or monotone surrogates, and the practical cost depends on the application.
All conditional coverage results are \emph{slice-conditional},
weaker than pointwise conditional
coverage~\cite{barber2021limits}.
On non-singular but non-isometric charts (bounded distortion
everywhere), the same redistribution mechanism applies but with
smaller magnitude; geodesic scoring still eliminates it.
Finally, geodesic scoring eliminates chart-induced distortion
but does not address heteroscedastic model error, which requires
complementary methods such as normalised conformal prediction.

\section{Conclusion}
\label{sec:conclusion}

We studied conformal prediction for manifold-valued vision outputs and showed that chart-coordinate nonconformity scores can silently redistribute coverage. Across four datasets, slice-conditional coverage drops by 30--50\pp{} near coordinate singularities despite nominal marginal coverage. We formalised this mechanism and proved that scalar adaptive conformal methods can adjust set size but cannot correct chart-induced shape distortion.

Replacing chart residuals with coordinate-free geodesic, chord, or quaternion-equivalent scores consistently reduces this failure without retraining and with negligible overhead. For geometric outputs, conformal reliability therefore depends not only on calibration, but also on measuring errors in a geometry-consistent way.

\section*{Acknowledgements}
This work was supported by the UKRI BBSRC project EyeWarn (APP37953).

%
%
\bibliographystyle{splncs04}
\bibliography{main}
\end{document}